\documentclass{article}
\usepackage{spconf, amsmath, graphicx}
\input{mysymbol.sty}

\usepackage{amsfonts}
\usepackage{color, xcolor}
\usepackage{mathtools}

\usepackage{amsthm}

\newtheorem{theorem}{\hspace{0pt}\bf Theorem}

\usepackage{algorithmic}
\usepackage{algorithm}

\date{\today}

\def\E{\mathbb{E}}

\usepackage{caption}
\title{A State-Augmented Approach for Learning \\ Optimal Resource Management Decisions in Wireless Networks}
\name{Yi\u{g}it~Berkay~Uslu$^{\star}$ \qquad Navid NaderiAlizadeh$^{\star}$ \qquad Mark Eisen$^{\dagger}$ \qquad Alejandro Ribeiro$^{\star}$}
\address{$^{\star}$ University of Pennsylvania \\
$^{\dagger}$ Intel Corporation}

\begin{document}
\ninept
\maketitle
\begin{abstract}
We consider a radio resource management (RRM) problem in a multi-user wireless network, where the goal is to optimize a network-wide utility function subject to constraints on the ergodic average performance of users. We propose a state-augmented parameterization for the RRM policy, where alongside the instantaneous network states, the RRM policy takes as input the set of dual variables corresponding to the constraints. We provide theoretical justification for the feasibility and near-optimality of the RRM decisions generated by the proposed state-augmented algorithm. Focusing on the power allocation problem with RRM policies parameterized by a graph neural network (GNN) and dual variables sampled from the dual descent dynamics, we numerically demonstrate that the proposed approach achieves a superior trade-off between mean, minimum, and 5\textsuperscript{th} percentile rates than baseline methods.
\end{abstract}
\begin{keywords}
Wireless networks, radio resource management, graph neural networks, Lagrangian duality, state augmentation.
\end{keywords}
\section{Introduction}
\label{sec:intro}

With the widespread adoption of 5G standards across the globe and 6G research already underway to meet the demands for increasing data rates in denser networks, significant efforts are being focused on the problem of \emph{radio resource management} (RRM), where the goal is to optimize the allocation of limited resources (over time, frequency, etc.) across wireless networks. While there are a plethora of optimization and information theory based traditional approaches in the literature to address the RRM problem ~\cite{shi2011iteratively, wu2013flashlinq, naderializadeh2014itlinq, yi2015itlinq+,shen2017fplinq}, learning-based methods have recently gained considerable traction and demonstrated superior performance over prior approaches~\cite{eisen2019learning, nasir2019multi, liang2019deep, eisen2020optimal, shen2020graph, naderializadeh2021resource}, thanks to the recent success of machine/deep learning and ubiquity of data and  computational resources both at the end-user devices and within the network infrastructure~\cite{niknam2020intelligent,lee2020graph, bonati2021intelligence,chowdhury2021unfolding,wang2021unsupervised,letaief2021edge,zhao2021link, nikoloska2021modular, MediaTek_6G_whitepaper_2022,li2022power}.

In~\cite{eisen2019learning, eisen2020optimal, ribeiro2012optimal, naderializadeh2022learning}, a general class of RRM problems are considered, where the goal is to maximize a network-wide utility function, subject to a set of constraints, all of which are defined based on the long-term average performance of users across the wireless network. A common method for solving such problems is to move to the Lagrangian dual domain, where a single objective, i.e., the Lagrangian, can be maximized over the primal variables and minimized over the dual variables, with each dual variable corresponding to a constraint in the original RRM problem. Although the dual problem has been shown under mild assumptions to exhibit a negligible duality gap when the infinite-dimensional RRM policies are approximated by near-universal parameterizations (e.g., neural networks), thus enabling use of primal-dual methods to find near-optimal solutions, it is unclear whether the found RRM policies lead to feasible decisions that satisfy the original constraints.

In this paper, we propose a \emph{state-augmented} RRM algorithm based on the notion of state augmentation in~\cite{calvo2021state}, where the wireless network state at each time step is augmented by the corresponding dual variables as RRM policy inputs. The dual variables are informative for making optimal RRM decisions, since they serve as indicators of how much the constraints have been violated or satisfied over time. Under certain assumptions, we theoretically show that the proposed state-augmented algorithm, unlike conventional primal-dual methods, generates a trajectory of RRM decisions that are both feasible and near-optimal. We also propose a method to improve the offline training procedure by running the dual dynamics in the background during the training phase and sampling the dual variable inputs to the state-augmented RRM policies from the set of optimal dual variables collected in the prior training epochs. 
Moreover, we train a secondary regression model using supervised learning with training network realizations as input and the corresponding optimal dual variable estimates as output, and we leverage this model to initialize the dual variables for the unseen network realizations during the execution phase close to their optimal values for faster convergence of the dual iterates.

We apply the aforementioned state-augmented algorithm to a power control RRM problem with the goal of maximizing the network sum-rate, subject to per-user minimum-rate requirements. We model the network state as the set of channel gains at each time step, and use a graph neural network (GNN) parameterization for the RRM policy. This GNN is trained to output the transmit power levels at each time step by processing the network state as the input graph edge weights and the dual variables as the input graph node features. Through extensive simulations, we show the superiority of the proposed state-augmented algorithm over baseline RRM methods in terms of the trade-off between the mean, minimum, and $5^{th}$ percentile rates. We also present ablation results, which demonstrate that running the dual dynamics in the training phase and sampling the dual variables from the observed trajectories, as well as using the secondary regression model to initialize the dual variables during execution lead to better trade-offs if not superior performance.

\section{Problem Formulation}\label{sec:formulation}
We consider a wireless network 
that operates over a series of time steps $t\in\{0,1,2,\dots,T-1\}$, where at each time step $t$, the set of channel gains in the network, or the \emph{network state}, is denoted by $\bbH_t \in \ccalH$. We assume that the channel state at time step $t$ can be decomposed into a long-term fading component $\bbH^l \in \ccalH^l$ that is assumed to be constant throughout the $T$ time steps and a short-term fast fading component $\bbH^{s}_{t} \in \ccalH^s$. Given the network state, we let $\bbp(\bbH_t)$ denote the vector of \emph{radio resource management (RRM)} decisions across the network, where $\bbp : \ccalH \to \reals^a$ denotes the RRM function. These RRM decisions subsequently lead to the network-wide performance vector $\bbf(\bbH_t, \bbp(\bbH_t)) \in \reals^b$, with $\bbf: \ccalH \times \reals^a \to \reals^b$ denoting the performance function.

Given a concave utility $\mathcal{U}: \reals^b \to \reals$ and a set of $c$ concave constraints $\bbg: \reals^b \to \reals^c$, we define the generic RRM problem as
\begin{subequations}\label{eq:nonparam_problem}
\begin{alignat}{2}
    &\max_{\{\bbp(\bbH_t)\}_{t=0}^{T-1}} &~~& \mathcal{U}\left( \frac{1}{T} \sum_{t=0}^{T-1} \bbf(\bbH_t, \bbp(\bbH_t)) \right),\label{eq:objective_non_param}             \\
    &~~~~~~\text{s.t.} &&  \bbg\left( \frac{1}{T} \sum_{t=0}^{T-1} \bbf(\bbH_t, \bbp(\bbH_t)) \right) \geq \bbzero,%
\end{alignat}
\end{subequations}
where the objective and the constraints are derived based on the \emph{ergodic average} network performance $\frac{1}{T} \sum_{t=0}^{T-1} \bbf(\bbH_t, \bbp(\bbH_t))$ rather than the instantaneous performance $\bbf(\bbH_t, \bbp(\bbH_t))$. The goal of the RRM problem is, therefore, to derive the optimal vector of RRM decisions $\bbp(\bbH_t)$ for any given network state $\bbH_t \in \ccalH$.

\section{Proposed State-Augmented RRM Algorithm}\label{sec:alg}
We propose to replace 
the infinite-dimensional, constrained functional optimization problem of~\eqref{eq:nonparam_problem} with a parametrized Lagrangian dual problem. In particular, by inputting both the network state $\bbH_t$ and the corresponding dual variables to the RRM policy at each time step $t$, the resulting RRM decisions $\bbp(\bbH)$ are represented by a finite-dimensional parametrization given by $\bbp^{\bbphi}(\bbH, \bbmu; \bbphi)$, where $\bbphi\in\bbPhi$ denotes the set of parameters of the state-augmented RRM policy.

For a set of dual variables $\bbmu \in \reals_+^c$, we define the \emph{augmented} Lagrangian as
\begin{align}
\ccalL_{\bbmu}(\bbphi) &= \mathcal{U}\left( \frac{1}{T} \sum_{t=0}^{T-1} \bbf(\bbH_t, \bbp^{\bbphi}(\bbH_t, \bbmu;\bbphi)) \right) \nonumber \\
&\quad+ \bbmu^T \bbg\left( \frac{1}{T} \sum_{t=0}^{T-1} \bbf(\bbH_t, \bbp^{\bbphi}(\bbH_t, \bbmu;\bbphi)) \right).\label{eq:augmented_Lagrangian}
\end{align}
Then, considering a probability distribution $p_{\bbmu}$ for the dual variables, we define the optimal state-augmented RRM policy as that which maximizes the expected augmented Lagrangian over the distribution of all dual parameters, i.e., 
\begin{align}
\bbphi^{\star} &= \arg \max_{\bbphi \in \bbPhi} \E_{\bbmu \sim p_{\bbmu}}\left[ \ccalL_{\bbmu}(\bbphi) \right]. \label{eq:phi_dynamics_augmented}
\end{align}

Utilizing the state-augmented policy parameterized by $\bbphi^{\star}$ in \eqref{eq:phi_dynamics_augmented}, at each time step $t$, we can obtain the Lagrangian-maximizing RRM decisions given by $\bbp(\bbH_t, \bbmu_{\lfloor t/T_0 \rfloor}; \bbphi^{\star})$. To that end, we introduce dual iterates $\bbmu_{k}$, where $k\in\{0,1,2,\dots,K-1\}$ with $K=\lfloor T / T_0 \rfloor$, and $T_0$ denotes the time duration between consecutive dual variable updates. We perform the dual descent updates as 

\begin{align}
\bbmu_{k+1} \hspace{-3pt} &= \hspace{-3pt} \left[\bbmu_k - \eta_{\bbmu} \bbg\left( \frac{1}{T_0} \hspace{-6pt} \sum_{t=kT_0}^{(k+1)T_0-1} \bbf(\bbH_{t}, \bbp^{\bbphi}(\bbH_{t},\bbmu_k;\bbphi^{\star})) \right)\right]_+\hspace{-7pt},\label{eq:mu_dynamics_augmented}
\end{align}
where $[\cdot] := \max(\cdot, 0)$ ensures nonengativity of the dual variables and $\eta_{\bbmu}$ denotes the learning rate for the dual variables. Note that the dual descent dynamics in~\eqref{eq:mu_dynamics_augmented} track the satisfaction/violation of the original constraints. In particular, if the RRM decisions at time step $t$ help satisfy the constraints, the dual variables are reduced. Conversely, if the constraints are not satisfied at a given time step, the dual variables increase in value.

We state the following theorem, which establishes the feasibility and near-optimality of the outlined state-augmented RRM algorithm.

\begin{theorem}\label{thm:near_universality_result}
Under certain mild assumptions such as the boundedness and strict feasibility of the constraints, expectation-wise Lipschitzness of the utility function $\mathcal{U}$ and performance function $\bbf$, and the near-universality of the state-augmented parameterization $\bbp^{\bbphi}(\bbH, \bbmu; \bbphi)$ with degree $\epsilon$ (see Definition~1 in \cite{StateAugmented_RRM_GNN_naderializadeh2022}), the sequence of the RRM decisions made by the proposed state-augmented algorithm in ~\eqref{eq:phi_dynamics_augmented}-\eqref{eq:mu_dynamics_augmented} are both feasible, i.e.,
\begin{align}\label{eq:thm_feasibility_state_augmented}
\lim_{T\to\infty}\bbg\left( \frac{1}{T} \sum_{t=0}^{T-1} \bbf\left(\bbH_t, \bbp^{\bbphi}\left(\bbH_t, \bbmu_{\lfloor t/T_0 \rfloor}; \bbphi^{\star}\right) \right) \right) \geq \bbzero, \, a.s.
\end{align}
and near-optimal, i.e.,
\begin{align}
&\lim_{T\to\infty} \E\left[ \mathcal{U}\left( \frac{1}{T} \sum_{t=0}^{T-1} \bbf\left(\bbH_t, \bbp^{\bbphi}\left(\bbH_t, \bbmu_{\lfloor t/T_0 \rfloor}; \bbphi^{\star}\right) \right) \right)\right] \nonumber \\
&\geq P^{\star} - \ccalO(c\eta_{\mu}) - \ccalO(\epsilon).
\label{eq:thm_optimality_state_augmented}
\end{align}
\end{theorem}

\begin{proof}
See \cite{StateAugmented_RRM_GNN_naderializadeh2022}.
\end{proof}

It is noteworthy that the state-augmented parameterization $\bbp(\bbH, \bbmu; \bbphi)$ does not only generate a trajectory of feasible and near-optimal RRM decisions, but it also obviates any requirement of memorizing optimal model parameter and dual variable pairs $(\bbtheta^{\star}, \bbmu^{\star})$ (see~\cite{StateAugmented_RRM_GNN_naderializadeh2022} for more details).

\begin{algorithm}[!t]
\caption{Training the State-Augmented RRM Algorithm Model Parameters}
    \label{alg:training}
    \begin{algorithmic}[1]
    \STATE {\bfseries Input:} $N$, $B$, $T$, $T_0$, $K_0$, $N_0$, $N_{\text{start}}$, $N_{\text{end}}$, $\eta_{\bbphi}$, $\eta_{\bbpsi}$, $\eta_{\bbmu}$.
    \STATE Initialize: $\bbphi_0, \, p_{\bbmu, 0};\, k \gets 0$.
    \FOR{$n=0, \ldots, N-1$}
        \FOR{$b=0, \ldots, B-1$}
            \STATE
            Randomly sample $\bbmu_{b,n} \sim p_{\bbmu, n}$ and let $\tilde{\bbmu}_{b, n, 0} \gets \bbmu_{b, n}$.
            \STATE
            Randomly generate a sequence of network states $\{\bbH_{b,t}\}_{t=0}^{T-1} = \{\bbH^{l}_{b} \bbH^{s}_{b,t}\}_{t=0}^{T-1}$.
            \FOR{$t=0, \ldots, T-1$}
                \STATE 
                Generate RRM decisions $\bbp^{\bbphi}\left(\bbH_{b,t}, \bbmu_{b,n}; \bbphi_n\right)$.
                \IF{$t+1 \mod T_0 = 0$}
                    \STATE
                    Run the dual dynamics of ~\eqref{eq:mu_dynamics_augmented} in the background with $\bbmu_k$ replaced by $\tilde{\bbmu}_{b, n, k+1}$ to observe dual variable trajectories.
                    \STATE
                    $k \gets k+1$
                \ENDIF
            \ENDFOR
            \STATE
            Calculate the augmented Lagrangian according to~\eqref{eq:augmented_Lagrangian}, i.e., $\ccalL_{\bbmu_{b}}(\bbphi_n) = \mathcal{U}\left( \! \frac{1}{T}\! \sum_{t=0}^{T-1} \! \bbf\left(\bbH_{b,t}, \bbp^{\bbphi}\left(\bbH_{b,t}, \bbmu_{b,n}; \bbphi_n\right)\right) \right)$ \\
            $+ {\bbmu^T_{b,n}} \bbg\left( \frac{1}{T} \sum_{t=0}^{T-1}\bbf\left(\bbH_{b,t}, \bbp^{\bbphi}\left(\bbH_{b,t}, \bbmu_{b,n}; \bbphi_n\right)\right) \right).$
        \ENDFOR
        \STATE
        Update the model parameters according to~\eqref{eq:phi_sga}.
        \IF{$n \geq N_{\text{start}}$ \AND $n < N_{\text{end}}$}
            \STATE
            Update the dual variable sampling distribution according to~\eqref{eq:dual_avg}-\eqref{eq:p_mu_update}.
        \ENDIF
    \ENDFOR
    \STATE
    $\bbphi^{\star} \gets \bbphi_N$.
    \STATE
    Train the dual regression model according to \eqref{eq:dual_regression_loss} to obtain $\bbpsi^{\star}$.
    \STATE
    {\bfseries Return:} {Optimal model and dual regression parameters $\bbphi^{\star}$, $\bbpsi^{\star}$.}
    \end{algorithmic}
\end{algorithm}


\subsection{Offline Training Procedure}
In order to tackle the maximization in~\eqref{eq:phi_dynamics_augmented} during the offline training phase, we iteratively solve the empirical version of~\eqref{eq:phi_dynamics_augmented} through gradient-ascent updates on the RRM model parameters given by
\begin{align}
\bbphi_{n+1} &= \bbphi_n +  \frac{\eta_{\bbphi}}{B} \sum_{b=0}^{B-1} \nabla_{\bbphi} \ccalL_{\bbmu_{b,n}}(\bbphi_n),\label{eq:phi_sga}
\end{align}
where $n\in\{0,1,2,\dots,N-1\}$ represents the policy parameter training iterations and $\eta_{\bbphi}$ denotes the learning rate corresponding to the RRM model parameters $\bbphi$. We denote the resulting optimal RRM model parameters as $\bbphi^{\star} \coloneqq \bbphi_N$. At the beginning of each iteration $n$, a batch of dual variables $\{\bbmu_{b,n}\}_{b=1}^B$ is sampled from a (discrete/continuous) distribution $p_{\bbmu, n}$ and for each batch, random realizations of the channel states $\{\bbH_{b,t}\}_{t=0}^{T-1} = \{\bbH^{l}_{b} \bbH^{s}_{b,t}\}_{t=0}^{T-1}$ are generated. We note the explicit dependence of the sampling distribution on the iteration index $n$ since we modify the training procedure outlined in \cite{StateAugmented_RRM_GNN_naderializadeh2022} by the following subroutine:

\begin{itemize}
    \item We run the dual dynamics in~\eqref{eq:mu_dynamics_augmented} \emph{in the background} during the training procedure to observe the trajectories followed by the dual variables, which we denote by $\tilde{\bbmu}$. We consider an update window for the initial dual variable sampling distribution, comprised of the iterations between $N_{\text{start}}$ and $N_{\text{end}}$. For all iterations $n < N_{\text{start}}$ (resp., $n \geq N_{\text{end}}$), we set $p_{\bbmu, n} = p_{\bbmu, 0}$ (resp., $p_{\bbmu, n} = p_{\bbmu, N_{\text{end}}}$).
    \item Throughout the update window $N_{\text{start}} \leq n < N_{\text{end}}$, we compute the average of the corresponding last $K_0$ dual variable iterates $\tilde{\bbmu}_{K-K_0+1}, \ldots, \tilde{\bbmu}_{K}$ over the most recent $N_0$ iterations for all network configurations and update $p_{\bbmu, n+1}$ as the distribution engendered by those averages. With a slight abuse of notation, we can express these averaging and updating operations by
    \begin{align}
        \overline{\bbmu}_{b,n}
        &= \left \{ \frac{1}{N_0 K_0}\sum_{m=n-N_0+1}^{n} \sum_{k=K - K_0+1}^{K} \tilde{\bbmu}_{b,m,k} \right \} \label{eq:dual_avg} \\
        p_{\bbmu, n+1} &\gets \left \{ \overline{\bbmu}_{b, n}
        \right \}_{b=0}^{B-1} \label{eq:p_mu_update}
    \end{align}
    
    \item Assuming $p_{\bbmu, N_{\text{end}}}$ converges to the optimal dual variable distribution $p_{\bbmu}^{\star}$, we train a secondary model $\bbd(\bbH^{l}_{b};\bbpsi)$, parameterized by a set of parameters $\bbpsi \in \bbPsi$, which we refer to as the \emph{dual regression model} (a similar model has also been proposed in~\cite{elenter2022a}). This model takes as input the long-term fading component, $\bbH^{l}_{b}$, and tries to predict the corresponding optimal dual variables. In particular, we train this model to minimize the mean squared error given by
    \begin{align}\label{eq:dual_regression_loss}
      \hspace{-3em}  \bbpsi^{\star} = \argmin_{\bbpsi \in \bbPsi} \frac{1}{B c} \sum_{b=0}^{B-1} \left\| \bbd(\bbH^{l}_{b};\bbpsi) - \overline{\bbmu}_{b, N_{\text{end}} + N_0} \right\|_{2}^{2}.
    \end{align}
    which can be performed via gradient-descent based approaches with a learning rate of $\eta_{\bbpsi}$.
\end{itemize}

A compact summary of the state-augmented offline training procedure is shown in Algorithm~\ref{alg:training}. We remark that although the modified training procedure samples the dual variables from the dual descent trajectories and hence is more robust to the choice of the initial sampling distribution, our experiments indicate that the choice of the initial sampling distribution still affects the training of the optimal RRM model parameters.

\subsection{Online Execution Phase}
The execution phase is identical to that of \cite{StateAugmented_RRM_GNN_naderializadeh2022} except here, we initialize the dual variables according to the output of the dual regression model, $\bbmu_0 = \bbd(\bbH^{l}; \bbpsi^{\star})$. During execution, the dual dynamics are used to update the dual variables and generate the RRM decisions as follows: For any time step $t\in\{0,1,2,\dots,T-1\}$, given the network state $\bbH_t$, we generate the RRM decisions using the state-augmented RRM policy $\bbp(\bbH_t, \bbmu_{\lfloor t / T_0 \rfloor};\bbphi^{\star})$. Then, we update the dual variables every $T_0$ time steps as in~\eqref{eq:mu_dynamics_augmented}.

\begin{figure*}[t]
    \centering
    \includegraphics[width=\textwidth]{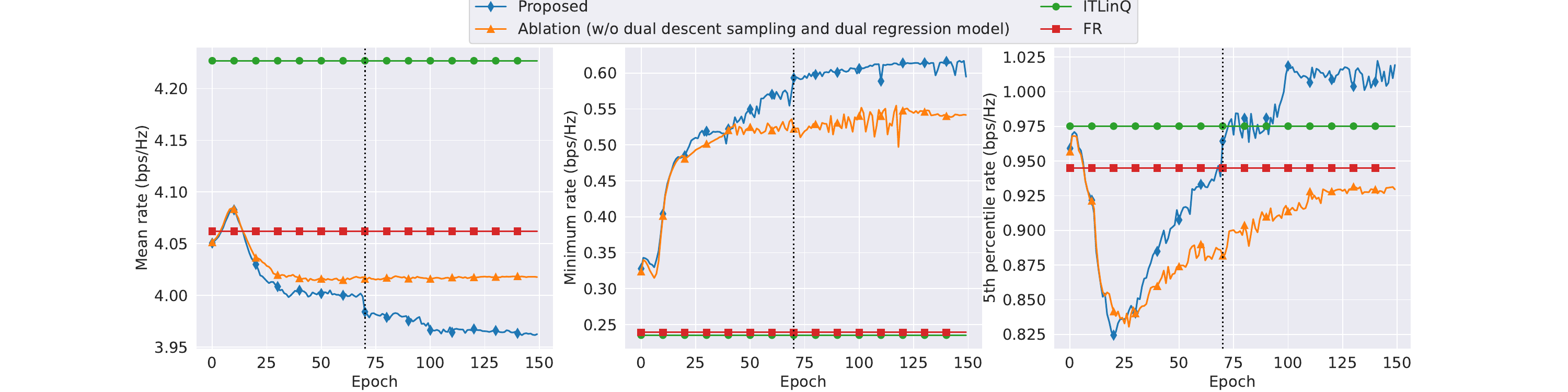}
    \caption{Convergence behavior of the proposed state-augmented RRM algorithm with dual descent sampling and dual regression model training and its comparison with its counterpart without dual descent sampling and dual regression model and baseline methods for $m=12$ transmitter-receiver pairs, fixed-density scenario, $f_{\min} = 0.5$ bps/Hz. The dashed vertical line coincides with the first epoch when the dual variable initializations using the dual regression model begin.}
\label{fig:all_rates_single_network}
\end{figure*}

\begin{figure*}[t]
    \centering
    \includegraphics[width=\textwidth]{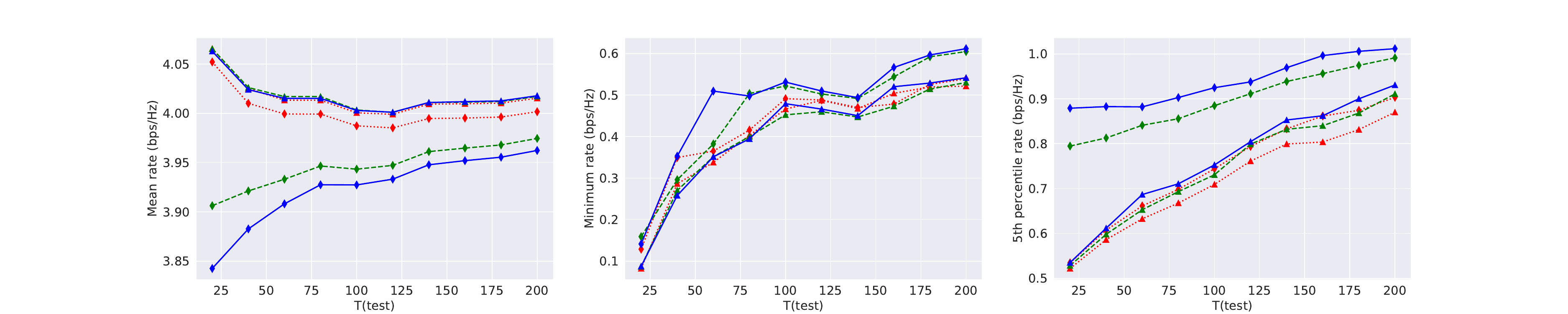}
    \caption{Comparison of the proposed and ablated state-augmentation algorithms in terms of the time evolution of the performance metrics at three different epochs for the same network configuration as in Figure~\ref{fig:all_rates_single_network}. Dotted, dashed and solid lines correspond to the model evaluations at training epochs of $N/3 = 50$, $2N/3 = 100$ and $N-1=149$, respectively, whereas the diamond and triangular markers identify the curves for the proposed and ablated state-augmentation algorithms.}
\label{fig:all_rates_vs_T_single_network}
\end{figure*}

\section{Experimental Results}
\label{sec:power_control}

We put the proposed state-augmented learning algorithm to test in the power control problem setting of \cite{StateAugmented_RRM_GNN_naderializadeh2022}. To summarize the setup briefly, the networks are comprised of $m$ \emph{users}, i.e., transmitter-receiver pairs $\{(\mathsf{Tx}_i, \mathsf{Rx}_i)\}_{i=1}^m$, where each transmitter communicates with a single designated receiver and each receiver treats the incoming signals from all but its associated transmitter as interference. The network state at time step $t$, i.e., $\bbH_t\in\mathbb{C}^{m \times m}$, contains all channel gains in the network, where the channel gain between transmitter $\mathsf{Tx}_i$ and receiver $\mathsf{Rx}_j$ at time step $t$ is denoted by by $h_{ij, t}\in\mathbb{C}$. Denoting the maximum transmit power by $P_{\max}$, the RRM decisions $\bbp \in [0, P_{\max}]^m$ represent the transmit power levels of the transmitters. The performance function $\bbf(\bbH_t, \bbp)\in\reals^m$ represents the receiver rates, where for each receiver $\mathsf{Rx}_i$, the rate at time step $t$ is given by
\begin{align}
f_i(\bbH_t, \bbp) = \log_2\left(1+\frac{p_i \left|h_{ii,t}\right|^2}{N + \sum_{j=1, j\neq i}^m p_j \left|h_{ji,t}\right|^2}\right),
\end{align}
where $N$ denotes the noise variance, it is assumed that capacity-achieving codes are used and the receiver treats all incoming interference as noise. We use a sum-rate utility $\ccalU(\bbx)=\sum_{i=1}^m x_i$, and $m$ constraints $g_i(\bbx) = x_i - f_{\min}, i=1,\ldots, m$ where $f_{\min}$ denotes the minimum per-user rate requirement.

Motivated by their scalability, transferability and permutation-equivariance properties, we parameterize both the state-augmented RRM policies and the dual regressor by graph neural network architectures as in~\cite{eisen2020optimal, lee2020graph, naderializadeh2022learning, shen2019graph}. We adopt the same 3-layered architecture as in~\cite{StateAugmented_RRM_GNN_naderializadeh2022} for the state-augmented model where at each time step, the nodes represent users (i.e., transmitter-receiver pairs), the edges between nodes $i$ and $j$, $e_{ij}$ are weighted and normalized by a function of the channel gains $h_{ij}$ and the input node features are the corresponding dual variables. The output layer produces a scalar output feature for each node in unit interval, which is scaled by $P_{\max}$ to obtain the transmit power decisions that satisfy the power level constraint of  $\bbp \in [0, P_{\max}]$. The dual regression model shares the same backbone as the RRM model but takes only the long-term channel component $\bbH_l$ as the input and no node features (or equivalently, dummy node features). 

We briefly mention the experiment parameters, majority of which carry over from~\cite{StateAugmented_RRM_GNN_naderializadeh2022} unchanged. We consider dual-slope path-loss model with log-normal shadowing for the large-scale fading and Rayleigh distribution for the small-scale fading. Unless otherwise noted, we set $P_{\max} = 10$dBm, $N=-104$dBm, $R = \sqrt{m / 20} \times 1$km and $R=500$m for fixed and variable density configurations, respectively, $T_0 = 5$, $K_0 = 5$, $T = 200$, $N_0 = 10$, $N = N_{\bbphi} = 150$ and $N_{\psi} = 50$ training epochs for the state-augmented RRM and dual regression models, respectively, $N_{\text{start}} = N / 10 = 15$, $N_{\text{end}} = 0.4N = 60$. In addition, we set the learning rates as $\eta_{\bbphi} = 10^{-1} / m$, $\eta_{\bbpsi} = 10^{-3}$, $\eta_{\bbmu} = 2$. Using a batch size of $B=128$, a total of $256$ training samples and $128$ test samples are generated. We initially draw the dual variables randomly from $U(0,1)$ distribution. 

We compare the performance of our proposed state-augmented RRM algorithm in terms of the mean rate, minimum rate and 5\textsuperscript{th} percentile rate metrics against full reuse (FR), where every transmitter uses $P_{\max}$, ITLinQ~\cite{naderializadeh2014itlinq} and the ablated state-augmented algorithm lacking the dual descent sampling and dual variable initialization by the dual regression model subroutines.

Figures~\ref{fig:all_rates_single_network} and~\ref{fig:all_rates_vs_T_single_network} compare the performance of the proposed algorithm with the baselines for $m=12$ users in a fixed-density scenario with $f_{\min} = 0.5$. As the first figure shows, both state-augmented algorithms meet the minimum rate constraint at the expense of slightly lower mean rates while the other two baselines fail to satisfy the minimum rate requirement. Nevertheless, the proposed state-augmentation algorithm surpasses its ablated counterpart by a noticeable margin with respect to the percentile rate metric, which leads to a more fair policy, and delivers even higher minimum rate than required by the minimum rate constraint. The second figure compares the execution phase convergence behavior of the proposed and ablated state-augmentation algorithms with RRM models trained for three separate training iterations, first of which is before the dual regression model is trained. It is readily seen in the minimum and percentile rate plots that the suboptimal, transient period lasts noticeably shorter for the proposed state-augmentation algorithm than for the ablated algorithm. That the gap between the mean rates in earlier time steps is much larger and becomes narrower as time rolls out is to be expected since when the dual variables are not initialized (close) to zero, the relative contribution of the utility term in~\eqref{eq:augmented_Lagrangian} to the maximization of the Lagrangian is lowered.

\section{Conclusion}\label{sec:conclusion}
We considered the radio resource management (RRM) problem in multi-user wireless networks where the goal is to maximize a network-wide utility subject to constraints on the ergodic average performance of the users. Based on the notion of state-augmentation where RRM policies are parametrized not only by the network states but also by the dual variables as an indicator of constraint violations, we proposed a state-augmented algorithm which is proven to lead to RRM decisions that are feasible and approximately optimal given a near-universal parametrization. Using graph neural network (GNN) architectures to parametrize both the state-augmented RRM policies for optimal power allocations and a dual regression model for estimating the optimal dual variables based on the dual descent trajectory samples, the proposed method is shown via numerical experiments to achieve superior trade-off between the mean rate, minimum rate and percentile rate than the baseline methods.




\clearpage

\bibliographystyle{IEEEbib}
\bibliography{refs}

\end{document}